\documentclass[11pt]{article}
\usepackage[T1]{fontenc}
\usepackage{lmodern}
\usepackage[margin=0.85in]{geometry}
\usepackage{amsmath,amssymb,amsthm}
\usepackage{microtype}
\usepackage{booktabs,tabularx,array}
\usepackage[hidelinks]{hyperref}
\newtheorem{proposition}{Proposition}
\newcommand{\R}{\mathbb R}
\newcommand{\E}{\mathbb E}
\newcommand{\normtwo}[1]{\lVert#1\rVert_2}
\newcommand{\normop}[1]{\lVert#1\rVert_{\mathrm{op}}}
\newcommand{\normnuc}[1]{\lVert#1\rVert_*}
\DeclareMathOperator{\Var}{Var}
\DeclareMathOperator{\Cov}{Cov}
\newcommand{\ip}[2]{\langle #1,#2\rangle_{\mathrm F}}
\title{Notes on Uniform Policy Score Bounds and Lipschitz Continuity for Truncated Gaussian, Product Cauchy, Boltzmann, and Softmax Policies}
\author{Yanjie Dong @ Xi'an Jiaotong University}
\date{}
\begin{document}
\maketitle
\vspace{-1.5em}

\begin{abstract}
We establish uniform score bounds and Lipschitz continuity for truncated Gaussian, Cauchy, Boltzmann, and softmax policies under explicitly stated parameterizations.
\end{abstract}

\section{The two regularity requirements}
For a policy with a matrix parameter $W$, define its score by
$S_W(s,a):=\nabla_W\log\pi_W(a\mid s)$, where the gradient uses the
Frobenius inner product $\ip{U}{V}=\operatorname{tr}(U^\top V)$.
We seek constants $R_\pi,L_\pi<\infty$ such that
\begin{align}
\normnuc{S_W(s,a)}&\le R_\pi,\label{eq:target-bounded}\\
\normnuc{S_W(s,a)-S_{W'}(s,a)}
&\le L_\pi\normop{W-W'} .\label{eq:target-lipschitz}
\end{align}

Both inequalities must hold uniformly over states, supported actions,
and the full stated parameter space. The score uses the \emph{nuclear norm};
the parameter difference uses the \emph{operator (spectral) norm}.
For a row or column matrix both norms equal its Euclidean norm; for a
scalar they equal absolute value. The dimensions of $W$ are specified
separately for each policy. The symbol $B$ is reused for the applicable
uniform feature bound in each policy class.

The following sufficient conditions verify both requirements without
restricting the parameter matrix to a compact set. They are claims about
particular parameterizations, not about every policy bearing the same
family name. Throughout, features, supports, scales, and temperatures are
independent of $W$ unless expressly stated otherwise. Softmax is the
finite-action instance of a Boltzmann policy, not a competing category.
The following sections establish the bounds for each policy class and
compare their assumptions.

\section{Truncated Gaussian with a linear mean}
Let $\mathcal A\subset\R^m$ be a fixed, bounded, Lebesgue-measurable
set of positive measure, with Euclidean diameter
\[
D:=\sup_{a,b\in\mathcal A}\normtwo{a-b}<\infty.
\]

Let $\Sigma\in\R^{m\times m}$ be fixed, symmetric, and positive definite.
The feature map
$\phi:\mathcal S\to\R^d$ is independent of the policy parameter and satisfies
$\sup_{s\in\mathcal S}\normtwo{\phi(s)}\le B<\infty$.

For $W\in\R^{m\times d}$, set $\mu_W(s)=W\phi(s)$ and define the
Lebesgue density
\begin{equation}\label{eq:density}
\begin{split}
\pi_W(a| s)
&=\frac{\mathbf 1_{\mathcal A}(a)
	\exp\!\left[-\tfrac12(a-\mu_W(s))^\top \Sigma^{-1}(a-\mu_W(s))\right]}
{Z(\mu_W(s))}, \\
Z(\mu) &:=\int_{\mathcal A}
\exp\!\left[-\tfrac12(x-\mu)^\top \Sigma^{-1}(x-\mu)\right]\,dx.
\end{split}
\end{equation}

This is the Gaussian law conditioned on membership in $\mathcal A$.
No convexity of $\mathcal A$ is assumed.
The normalizer satisfies $0<Z(\mu)<\infty$ for every $\mu\in\R^m$.

Let $\pi_\mu$ denote the density in \eqref{eq:density} with mean parameter
$\mu$, and write
\begin{equation}\label{eq:moments}
m(\mu):=\E_\mu[X],
\qquad
C(\mu):=\E_\mu[(X-m(\mu))(X-m(\mu))^\top].
\end{equation}

The score is the matrix gradient
$S_W(s,a):=\nabla_W\log\pi_W(a\mid s)\in\R^{m\times d}$,
defined through the Frobenius inner product. All score claims below concern
$a\in\mathcal A$; the logarithm of the density is not finite outside this set.
We use the nuclear norm $\normnuc{\cdot}$ for score matrices, the operator
(spectral) norm $\normop{\cdot}$ for parameter matrices, and the Euclidean
norm $\normtwo{\cdot}$ for vectors.

\begin{proposition}[Uniform boundedness and global score Lipschitz continuity]
\label{prop:score}
Under the preceding assumptions, for every $W,W'\in\R^{m\times d}$,
$s\in\mathcal S$, and $a\in\mathcal A$,
\begin{align}
S_W(s,a)
&=\Sigma^{-1}\bigl(a-m(\mu_W(s))\bigr)\phi(s)^\top,
\label{eq:score}\\
\normnuc{S_W(s,a)}
&\le R_\pi,
\qquad R_\pi:=BD\normop{\Sigma^{-1}},
\label{eq:bounded}\\
\normnuc{S_W(s,a)-S_{W'}(s,a)}
&\le L_\pi\normop{W-W'},
\qquad L_\pi:=\frac{B^2D^2}{4}\normop{\Sigma^{-1}}^{\,2}.
\label{eq:lipschitz}
\end{align}

These are sufficient constants, with no claim of sharpness. In particular,
no bound on $\normop{W}$ or on $\normtwo{\mu_W(s)}$ is required.
\end{proposition}

\begin{proof}
Fix a state $s$ and abbreviate $\phi=\phi(s)$.
On a compact set of mean parameters, the Gaussian kernel and its derivatives
are uniformly bounded on the bounded set $\mathcal A$. Since $\mathcal A$
has finite measure, differentiation under the integral is justified by
dominated convergence. Positivity of $Z$ then permits differentiation of
normalized moments and logarithms at every mean parameter.

We first differentiate the normalizer. The mean gradient of the unnormalized log kernel is $\Sigma^{-1}(x-\mu)$. Hence, we obtain
\begin{align}
\nabla_\mu Z(\mu)
&=\int_{\mathcal A}\Sigma^{-1}(x-\mu)
 \exp\!\left[-\tfrac12(x-\mu)^\top \Sigma^{-1}(x-\mu)\right]\,dx,
\nonumber\\
\nabla_\mu\log Z(\mu)
&=\Sigma^{-1}(m(\mu)-\mu).
\label{eq:normalizer}
\end{align}

For $a\in\mathcal A$, the support indicator is constant with respect to
the parameter. Subtracting \eqref{eq:normalizer} from the gradient of the
log kernel gives
\begin{equation}\label{eq:mean-score}
\nabla_\mu\log\pi_\mu(a)
=\Sigma^{-1}(a-\mu)-\Sigma^{-1}(m(\mu)-\mu)=\Sigma^{-1}(a-m(\mu)).
\end{equation}

The chain rule for $\mu=W\phi$ proves \eqref{eq:score}.
Thus the truncated mean, rather than the untruncated mean parameter,
centers the score.

We next differentiate the truncated mean. For a direction $h\in\R^m$, differentiation of the normalized expectation,
using \eqref{eq:mean-score}, gives
\begin{align}
D_\mu m(\mu)[h]
&=\E_\mu\!\left[X(X-m(\mu))^\top\right]\Sigma^{-1}h
=C(\mu)\Sigma^{-1}h,
\label{eq:mean-derivative}\\
D_W S_W(s,a)[U]
&=-\Sigma^{-1} C(\mu_W(s))\Sigma^{-1}\,U\phi\phi^\top,
\qquad U\in\R^{m\times d}.
\label{eq:score-derivative}
\end{align}

The matrix order in these identities matters:
$D_\mu m(\mu)=C(\mu)\Sigma^{-1}$.
In \eqref{eq:score-derivative}, $\Sigma^{-1}C\Sigma^{-1}$ is an $m\times m$ matrix,
$U\phi$ is an $m$-vector, and the output is an $m\times d$ matrix.
The mean/covariance derivative identity is also the standard
exponential-family identity under the natural parameter $\Sigma^{-1}\mu$;
see Fithian~\cite{fithian2023}. The norm constants below are derived here.

To control the covariance, we use the support diameter. If a scalar random variable $Y$ lies in $[\alpha,\beta]$ and
$\bar y=\E Y$, then
\[
0\le\E[(Y-\alpha)(\beta-Y)]
=(\beta-\bar y)(\bar y-\alpha)-\Var(Y).
\]

Consequently,
\begin{equation}\label{eq:scalar-variance}
\Var(Y)\le(\beta-\bar y)(\bar y-\alpha)
\le\frac{(\beta-\alpha)^2}{4}.
\end{equation}

For $v\in\R^m$, the range of $v^\top X$ has width at most
$D\normtwo{v}$. Applying \eqref{eq:scalar-variance} with the infimum
and supremum of this range yields
\begin{equation}\label{eq:covariance}
v^\top C(\mu)v=\Var_\mu(v^\top X)
\le\frac{D^2}{4}\normtwo{v}^{\,2},
\qquad
0\preceq C(\mu)\preceq\frac{D^2}{4}I_m.
\end{equation}

This argument does not require the extremal values to be attained.

It remains to bound the score and its derivative in the required norm. For vectors $u$ and $v$, the matrix $uv^\top$ has at most one nonzero
singular value, equal to $\normtwo{u}\normtwo{v}$. Therefore
\begin{equation}\label{eq:rank-one}
\normnuc{uv^\top}=\normtwo{u}\normtwo{v}.
\end{equation}

For every $a\in\mathcal A$,
\[
\normtwo{a-m(\mu)}
=\normtwo{\E_\mu[a-X]}
\le\E_\mu\normtwo{a-X}\le D.
\]

This remains valid for a nonconvex support, even if its mean is not in
$\mathcal A$. Eqs. \eqref{eq:score} and \eqref{eq:rank-one} now give
\[
\normnuc{S_W(s,a)}
=\normtwo{\Sigma^{-1}(a-m(\mu_W(s)))}\normtwo{\phi}
\le BD\normop{\Sigma^{-1}},
\]
which proves \eqref{eq:bounded}.

Similarly, \eqref{eq:score-derivative}--\eqref{eq:rank-one} imply
\begin{align}
\normnuc{D_W S_W(s,a)[U]}
&=\normtwo{\Sigma^{-1}C(\mu_W(s))\Sigma^{-1}\,U\phi}\normtwo{\phi}
\nonumber\\
&\le\normop{\Sigma^{-1}}^{\,2}\normop{C(\mu_W(s))}
       \normop{U}\normtwo{\phi}^{\,2}
\nonumber\\
&\le\frac{B^2D^2}{4}\normop{\Sigma^{-1}}^{\,2}\normop{U}.
\label{eq:derivative-bound}
\end{align}

Finally, let $W_t=W'+t(W-W')$. The fundamental theorem of calculus
and the triangle inequality for the nuclear norm give
\begin{align*}
S_W(s,a)-S_{W'}(s,a)
&=\int_0^1 D_W S_{W_t}(s,a)[W-W']\,dt,\\
\normnuc{S_W(s,a)-S_{W'}(s,a)}
&\le\int_0^1 L_\pi\normop{W-W'}\,dt
=L_\pi\normop{W-W'}.
\end{align*}

All bounds are uniform in $s$, $a$, and the endpoints of this segment,
which proves \eqref{eq:lipschitz}.
\end{proof}

For a scalar specialization, let $\mathcal A=[\ell,u]\subset\R$, with $\ell<u$, and
$\Sigma=\sigma^2>0$. Write $W=w^\top\in\R^{1\times d}$, so that
$\mu_W(s)=w^\top\phi(s)$, and denote the truncated mean by
$m_\sigma(\mu)$. Then
\[
S_W(s,a)=\frac{a-m_\sigma(\mu_W(s))}{\sigma^2}\phi(s)^\top,
\qquad
\frac{d m_\sigma(\mu)}{d\mu}
=\frac{\Var_\mu(X)}{\sigma^2}.
\]

Proposition~\ref{prop:score} supplies the constants
\begin{equation}\label{eq:scalar-constants}
R_\pi=\frac{B(u-\ell)}{\sigma^2},
\qquad
L_\pi=\frac{B^2(u-\ell)^2}{4\sigma^4}.
\end{equation}

For $\mathcal A=[-1,1]$ and $B\le1$, valid choices are
\[
R_\pi=\frac{2}{\sigma^2},
\qquad L_\pi=\frac{1}{\sigma^4}.
\]

For row matrices, both the nuclear and operator norms equal the Euclidean norm of the row, consistently with the general statement. 
The proof uses a support and covariance independent of $W$, and a linear
mean with a fixed, uniformly bounded feature map. A nonlinear neural mean
requires additional bounds on its first and second parameter derivatives
to obtain analogous uniform score bounds. A learned covariance requires
additional derivative controls and nondegeneracy conditions; the present
constants do not cover that setting.
Changing the support with $W$
also invalidates the differentiation argument given here without further
analysis. In the scalar example, clipping an untruncated Gaussian sample
creates atoms at the endpoints; it is not the conditioned density in
\eqref{eq:density}. These sufficient bounds concern the stated policy class
and do not establish score regularity for an arbitrary neural policy.

\section{Product Cauchy with a linear location}
Let the action space be $\R^m$, and let the coordinates be conditionally
independent given the state. Fix scales $\gamma_1,\ldots,\gamma_m>0$.
As before, take a fixed feature map with $\normtwo{\phi(s)}\le B$,
$W\in\R^{m\times d}$, and $\mu_W(s)=W\phi(s)$. Define
\begin{equation}\label{eq:cauchy-density}
\pi_W(a\mid s)=\prod_{i=1}^m
\frac{\gamma_i}{\pi\bigl(\gamma_i^2+(a_i-\mu_{W,i}(s))^2\bigr)},
\qquad a\in\R^m.
\end{equation}

Each factor is the usual location--scale Cauchy density
\cite{nist-cauchy}. Here $\mu_W$ is a \emph{location}, not an expectation:
a Cauchy variable has no finite mean. No action moment is used below. These bounds are uniform over all $a\in\R^m$ and $W\in\R^{m\times d}$.

\begin{proposition}[Fixed-scale product Cauchy]
\label{prop:cauchy}
For \eqref{eq:cauchy-density}, the score is
\begin{equation}\label{eq:cauchy-score}
S_W(s,a)=
\left(\frac{2(a_i-(W\phi(s))_i)}
{\gamma_i^2+(a_i-(W\phi(s))_i)^2}\right)_{i=1}^m\phi(s)^\top,
\end{equation}
where the parenthesized entries form a column vector.
The two regularity requirements hold with
\begin{equation}\label{eq:cauchy-constants}
R_\pi=B\left(\sum_{i=1}^m\gamma_i^{-2}\right)^{1/2},
\qquad L_\pi=\frac{2B^2}{\min_i\gamma_i^2}.
\end{equation}

\end{proposition}
\begin{proof}
Fix $s$, abbreviate $\phi=\phi(s)$, and write
$r_i=a_i-\mu_i$ with $\mu=W\phi$ within this proof.
Differentiating the log density with respect to each location and applying
the chain rule gives \eqref{eq:cauchy-score}.

To bound the score, note that
$2\gamma_i|r_i|\le\gamma_i^2+r_i^2$. The rank-one identity therefore gives
\[
\normnuc{S_W(s,a)}
=\normtwo{\phi}\left(\sum_i
\left(\frac{2r_i}{\gamma_i^2+r_i^2}\right)^2\right)^{1/2}
\le B\left(\sum_i\gamma_i^{-2}\right)^{1/2}.
\]

For Lipschitz continuity, the derivative of each location-score component is
\begin{equation}\label{eq:cauchy-derivative}
\frac{\partial}{\partial\mu_i}
\frac{2(a_i-\mu_i)}{\gamma_i^2+(a_i-\mu_i)^2}
=-\frac{2(\gamma_i^2-r_i^2)}{(\gamma_i^2+r_i^2)^2}.
\end{equation}

Its absolute value is at most $2/(\gamma_i^2+r_i^2)\le2/\gamma_i^2$.
Thus, for any $U\in\R^{m\times d}$,
\begin{align*}
D_WS_W(s,a)[U]
&=\left(-\frac{2(\gamma_i^2-r_i^2)(U\phi)_i}
{(\gamma_i^2+r_i^2)^2}\right)_{i=1}^m\phi^\top,\\
\normnuc{D_WS_W(s,a)[U]}
&\le\frac{2}{\min_i\gamma_i^2}\normtwo{U\phi}\normtwo{\phi}
\le\frac{2B^2}{\min_i\gamma_i^2}\normop{U}.
\end{align*}

The output is an $m\times d$ matrix of rank at most one.
Integrating this uniform derivative bound along the segment from $W'$ to
$W$, as in Proposition~\ref{prop:score}, proves the Lipschitz inequality.
\end{proof}

For a scalar action, write $W=w^\top\in\R^{1\times d}$ and use one fixed
scale $\gamma$. Then
\[
S_W(s,a)=\frac{2(a-w^\top\phi(s))}
{\gamma^2+(a-w^\top\phi(s))^2}\phi(s)^\top,
\qquad R_\pi=\frac{B}{\gamma},\quad
L_\pi=\frac{2B^2}{\gamma^2}.
\]

With $m$ coordinates and a common scale, the bounded-score constant is
$B\sqrt m/\gamma$, whereas the Lipschitz constant remains
$2B^2/\gamma^2$ under the stated matrix norms.
Both calculations concern an \emph{independent product} of scalar Cauchy
laws. An elliptical multivariate Cauchy is instead a multivariate Student
law with one degree of freedom. When $m>1$,
even a diagonal shape matrix does not make its density factor into
independent scalar Cauchy densities. Its score must be derived from its
own density; the product constants are not asserted for that family.
The rational expressions in \eqref{eq:cauchy-score} and
\eqref{eq:cauchy-derivative} are bounded even when actions are unbounded.
Thus this policy satisfies both score requirements. This does not supply
finite action moments, bounded rewards, critic-feature moments, or any
other reinforcement-learning assumption; those require separate checks.

\section{Boltzmann policies with bounded matrix features}
For each state $s$, let $\nu_s$ be a parameter-independent positive measure
on an action space $\mathcal A$ with
$0<\nu_s(\mathcal A)<\infty$. Fix a temperature $\tau>0$ and a
measurable, parameter-independent feature matrix
$F(s,a)\in\R^{m\times d}$ satisfying
\begin{equation}\label{eq:boltzmann-feature-bound}
\sup_{s,a}\normnuc{F(s,a)}\le B<\infty.
\end{equation}

For $W\in\R^{m\times d}$, define the density with respect to $\nu_s$ by
\begin{align}
\pi_W(a\mid s)&=
\frac{\exp\!\left(\ip{W}{F(s,a)}/\tau\right)}{Z_s(W)},
\label{eq:boltzmann-density}\\
Z_s(W)&=\int_{\mathcal A}
\exp\!\left(\ip{W}{F(s,x)}/\tau\right)\nu_s(dx).
\label{eq:boltzmann-normalizer}
\end{align}

The exponent is a preference; under the negative-energy convention,
the energy is $-\ip{W}{F}$. A carrier density can be incorporated into
$\nu_s$ if its integral is positive and finite. Thus the action space
need not be bounded or discrete.

\begin{proposition}[Linear Boltzmann preference]
\label{prop:boltzmann}
Under these assumptions, \eqref{eq:boltzmann-density} is normalizable for
every $W$, and its score is
\begin{equation}\label{eq:boltzmann-score}
S_W(s,a)=\frac{F(s,a)-\E_W[F(s,A)]}{\tau}.
\end{equation}

Here $\E_W$ denotes expectation over an action drawn from the current
policy at the fixed state $s$. The two regularity requirements hold with
\begin{equation}\label{eq:boltzmann-constants}
R_\pi=\frac{2B}{\tau},\qquad
L_\pi=\frac{B^2}{\tau^2}.
\end{equation}

The score matrix need not have rank one.
\end{proposition}
\begin{proof}
Nuclear--spectral duality gives
$|\ip{W}{F(s,a)}|\le\normop{W}\normnuc{F(s,a)}$. Hence
\begin{equation}\label{eq:boltzmann-finite}
0<e^{-\normop{W}B/\tau}\nu_s(\mathcal A)
\le Z_s(W)\le e^{\normop{W}B/\tau}\nu_s(\mathcal A)<\infty.
\end{equation}

On every compact set of parameter matrices, the integrand and its first
two parameter derivatives are dominated by constants integrable with
respect to the finite measure $\nu_s$. Differentiation under the integral
is therefore justified. Differentiating the log normalizer yields
$\nabla_W\log Z_s(W)=\E_W[F(s,A)]/\tau$, proving
\eqref{eq:boltzmann-score}. Jensen's inequality and the triangle inequality
give $\normnuc{F(s,a)-\E_W[F(s,A)]}\le 2B$ and the bounded-score claim.

Fix $s$ and abbreviate the random feature matrix by $F=F(s,A)$.
All expectations, variances, and covariances below use the current policy.
Differentiating its normalized expectation in the direction
$U\in\R^{m\times d}$ gives
\begin{align}
D_W\E_W[F][U]
&=\frac{1}{\tau}\E_W[(F-\E_W F)\ip{U}{F-\E_W F}],
\label{eq:boltzmann-mean-derivative}\\
D_WS_W(s,a)[U]
&=-\frac{1}{\tau^2}\E_W[(F-\E_W F)\ip{U}{F-\E_W F}].
\label{eq:boltzmann-score-derivative}
\end{align}

Both outputs have the same $m\times d$ dimensions as $W$.
These are the exponential-family covariance identities
\cite{fithian2023}, here justified directly for bounded features.

To bound the derivative in the required norms, take
$V\in\R^{m\times d}$ and pair \eqref{eq:boltzmann-score-derivative} with $V$:
\begin{equation}\label{eq:boltzmann-dual-cov}
\ip{V}{D_WS_W(s,a)[U]}
=-\frac{1}{\tau^2}\Cov(\ip{V}{F},\ip{U}{F}).
\end{equation}

The scalar $\ip{U}{F}$ lies in
$[-B\normop{U},B\normop{U}]$. Thus \eqref{eq:scalar-variance} gives
$\Var(\ip{U}{F})\le B^2\normop{U}^{\,2}$, and the same argument applies to
$V$. Cauchy--Schwarz for covariance now implies
\[
\bigl|\ip{V}{D_WS_W(s,a)[U]}\bigr|
\le\frac{B^2}{\tau^2}\normop{V}\normop{U}.
\]

Finally, nuclear--spectral duality states that for every real matrix $X$,
\begin{equation}\label{eq:dual-norm}
\normnuc{X}=\sup_{\normop{V}\le1}|\ip{V}{X}|.
\end{equation}

Indeed, if $X=\sum_i\sigma_i p_iq_i^\top$ is a thin singular value
decomposition, then
$|\ip{V}{X}|\le\sum_i\sigma_i|p_i^\top Vq_i|
\le\normop{V}\sum_i\sigma_i$.
Equality is attained by $V=\sum_i p_iq_i^\top$; the zero matrix is immediate.
Taking this supremum proves
\[
\normnuc{D_WS_W(s,a)[U]}
\le\frac{B^2}{\tau^2}\normop{U}.
\]

Integration along the parameter segment proves the global Lipschitz
claim. This argument avoids both a rank assumption and a dimension factor.
\end{proof}

The Lipschitz constant in \eqref{eq:boltzmann-constants} cannot be reduced
using only \eqref{eq:boltzmann-feature-bound}: take a scalar parameter,
two equally weighted actions, and scalar features $F=\pm B$.
At $W=0$, the score derivative is $-B^2/\tau^2$.
Finite counting measure is suitable for a finite action set, while a
probability carrier is suitable on an unbounded continuous action space.
In contrast, with bounded $F$ and infinite-mass Lebesgue measure,
the positive lower bound on the integrand in
\eqref{eq:boltzmann-finite} makes the integral infinite for every finite
$W$. One cannot normalize such a bounded preference on that measure.
Other Boltzmann models with unbounded features or preferences require
their own normalizability and differentiation arguments.
If $W$ and $F$ are column vectors, the inner product and both norms
are Euclidean, so the same constants follow from
$\sup_{s,a}\normtwo{F(s,a)}\le B$.
For scalar parameters and features, all norms reduce to absolute values. A bound
only on the Frobenius norm of a general matrix feature is not the same
hypothesis: converting via
$\normnuc{F}\le\sqrt{\operatorname{rank}(F)}\lVert F\rVert_{\mathrm F}$
can introduce a rank factor in these constants.

\section{Finite-action softmax with linear logits}
Let $\mathcal A=\{1,\ldots,K\}$, let $\phi(s)\in\R^d$ satisfy
$\normtwo{\phi(s)}\le B$, and fix $\tau>0$. With
$W\in\R^{K\times d}$, write
\begin{equation}\label{eq:softmax-density}
p_i=\pi_W(i\mid s)=
\frac{\exp((W\phi(s))_i/\tau)}
{\sum_{j=1}^K\exp((W\phi(s))_j/\tau)}.
\end{equation}

This is Proposition~\ref{prop:boltzmann} with counting measure and
$F(s,i)=e_i\phi(s)^\top$, where $e_i$ is the $i$th coordinate vector in
$\R^K$. The additional structure yields sharper constants.

\begin{proposition}[Linear-logit softmax]
\label{prop:softmax}
For \eqref{eq:softmax-density},
\begin{equation}\label{eq:softmax-score}
S_W(s,a)=\frac{(e_a-p)\phi(s)^\top}{\tau}.
\end{equation}

The two regularity requirements hold with
\begin{equation}\label{eq:softmax-constants}
R_\pi=\frac{\sqrt{2}B}{\tau},\qquad
L_\pi=\frac{B^2}{2\tau^2}.
\end{equation}

These constants are independent of the number of actions. For $K=1$,
the score is identically zero and both constants may instead be zero.
\end{proposition}
\begin{proof}
Differentiating the log probability gives \eqref{eq:softmax-score}.
For any probability vector $p$,
\[
\normtwo{e_a-p}^{\,2}
=(1-p_a)^2+\sum_{i\ne a}p_i^2
\le2(1-p_a)^2\le2.
\]

The rank-one identity proves the bounded-score assertion.

For Lipschitz continuity, we consider the Jacobian of $p$ with respect to
the scaled logits $W\phi(s)/\tau$, which is
\begin{equation}\label{eq:softmax-jacobian}
\operatorname{diag}(p)-pp^\top.
\end{equation}

For $v\in\R^K$, $v^\top (\operatorname{diag}(p)-pp^\top)v$ is the variance of $v_I$ when
$\Pr(I=i)=p_i$, and is nonnegative. If $K\ge2$, its range has width at most
$\max_{i,j}|v_i-v_j|\le\sqrt{2}\normtwo{v}$. Applying
\eqref{eq:scalar-variance} gives
\begin{equation}\label{eq:softmax-jacobian-bound}
0\le v^\top (\operatorname{diag}(p)-pp^\top)v\le\tfrac12\normtwo{v}^{\,2},
\qquad\normop{(\operatorname{diag}(p)-pp^\top)}\le\tfrac12.
\end{equation}

For a direction $U\in\R^{K\times d}$, differentiating the score gives
\[
D_WS_W(s,a)[U]=-
\frac{(\operatorname{diag}(p)-pp^\top)U\phi(s)\phi(s)^\top}{\tau^2}.
\]

This $K\times d$ matrix has rank at most one, so
\begin{align*}
\normnuc{D_WS_W(s,a)[U]}
&=\frac{\normtwo{(\operatorname{diag}(p)-pp^\top)U\phi(s)}\normtwo{\phi(s)}}{\tau^2}\\
&\le\frac{B^2}{2\tau^2}\normop{U}.
\end{align*}

Integrating along the line segment between the parameters gives the
Lipschitz claim. The case $K=1$ is immediate.
\end{proof}

For two free logits, $p=(1/2,1/2)$ makes the matrix in \eqref{eq:softmax-jacobian} have an eigenvalue $1/2$;
thus the Jacobian factor in \eqref{eq:softmax-constants} is attained.

If instead there is one free binary logit with the other fixed to zero,
write $w\in\R^d$ and
\[
p=\pi_w(1\mid s)=\frac{e^{w^\top\phi/\tau}}{1+e^{w^\top\phi/\tau}},\qquad\pi_w(0\mid s)=1-p.
\]

For an action $y\in\{0,1\}$, the vector score is
$(y-p)\phi/\tau$, and its derivative is
$-p(1-p)\phi\phi^\top/\tau^2$. Consequently valid Euclidean constants are
\begin{equation}\label{eq:binary-constants}
R_\pi=\frac{B}{\tau},\qquad L_\pi=\frac{B^2}{4\tau^2}.
\end{equation}

These are constants for a different parameter space and norm; they do not
replace \eqref{eq:softmax-constants} for a matrix with two free rows.
For a directly parameterized vector of $K$ logits, take $d=1$ and
$\phi=1$ in Proposition~\ref{prop:softmax}. Both matrix norms are then
Euclidean and $B=1$.

\section{Comparison and use of the assumptions}

\begin{table}[htb]
\centering
\small
\setlength{\tabcolsep}{4pt}
\renewcommand{\arraystretch}{1.45}
\begin{tabularx}{\textwidth}{@{}>{\raggedright\arraybackslash}p{1.12in}>{\raggedright\arraybackslash}Xcc@{}}
\toprule
Policy & Fixed structure / uniform bound & $R_\pi$ & $L_\pi$\\
\midrule
Truncated Gaussian & $\mu=W\phi$, $\normtwo{\phi}\le B$;
$\Sigma$ fixed; support diameter $D$
& $BD\normop{\Sigma^{-1}}$ & $\dfrac{B^2D^2\normop{\Sigma^{-1}}^{\,2}}{4}$\\
Product Cauchy & $\mu=W\phi$, $\normtwo{\phi}\le B$;
scales $\gamma_i>0$
& $B\sqrt{\sum_i\gamma_i^{-2}}$ & $\dfrac{2B^2}{\min_i\gamma_i^2}$\\
Boltzmann & Linear preference $\ip{W}{F}$;
$\normnuc{F}\le B$; finite positive carrier mass
& $\dfrac{2B}{\tau}$ & $\dfrac{B^2}{\tau^2}$\\
Softmax & $K$ linear logits $W\phi$;
$\normtwo{\phi}\le B$; finite action set
& $\dfrac{\sqrt{2}B}{\tau}$ & $\dfrac{B^2}{2\tau^2}$\\
\bottomrule
\end{tabularx}
\caption{Uniform nuclear-score bounds and operator-to-nuclear score
Lipschitz constants. Temperatures, scales, covariance, features, and
supports are independent of the learned parameter. All constants are
sufficient; only the specific sharpness statements in the text are claimed.}
\label{tab:comparison}
\end{table}

Table~\ref{tab:comparison} collects sufficient constants for the precise
parameterizations proved above. All four rows satisfy both score
requirements; none by itself verifies the remaining hypotheses of an
actor--critic convergence theorem.
For a finite action set and a learned final linear layer, softmax gives a
direct verification of both requirements. For an unbounded continuous
action space, a fixed-scale product Cauchy with linear location also
satisfies both, provided any separate reward and moment assumptions are
appropriate. For a bounded continuous action space, the conditioned
Gaussian above supplies another verification. General Boltzmann policies
fit the result when their matrix sufficient statistic has a uniform
nuclear-norm bound and their carrier has finite positive mass.
Parameter-dependent features, nonlinear neural means or logits, learned
scales or covariance, changing supports, and learned temperatures are not
covered by the displayed constants. The required additional controls
depend on the parameterization: in particular, first and second parameter
derivatives enter score regularity. Vanishing scales or temperatures can
make the constants diverge.
A trainable preceding network makes its
features parameter-dependent; the linear-layer result cannot simply be
applied to the entire network. Fixed features may, however, be computed
by a frozen network if their uniform bound is justified.
These conditions concern the score pointwise in the state and action;
they do not establish mixing, critic accuracy, or finite variance of an
arbitrary policy-gradient estimator. Any use in a convergence theorem
must check those assumptions independently.


\begin{thebibliography}{3}
\bibitem{fithian2023}
Will Fithian.
\newblock \emph{Exponential Families}.
\newblock Statistics 210A course notes, University of California, Berkeley,
August 29, 2023.
\newblock Section \emph{Differential identities}.
\newblock \url{https://www.stat.berkeley.edu/~wfithian/courses/stat210a/exponential-families.html}.

\bibitem{nist-cauchy}
National Institute of Standards and Technology.
\newblock \emph{Cauchy Distribution}.
\newblock NIST/SEMATECH e-Handbook of Statistical Methods, Section 1.3.6.6.3.
\newblock \url{https://www.itl.nist.gov/div898/handbook/eda/section3/eda3663.htm}.
\end{thebibliography}
\end{document}